\documentclass{article}

\usepackage{microtype}
\usepackage{graphicx}
\usepackage{subfigure}
\usepackage{booktabs} 
\usepackage{amsmath, amssymb, amsthm}
\usepackage{hyperref}
\usepackage{geometry}
\usepackage{float}

\newtheorem{theorem}{Theorem}

\theoremstyle{definition}
\newtheorem{definition}{Definition}

\usepackage{amsthm}
\usepackage{algorithm}
\usepackage{algorithmic}

\theoremstyle{remark}
\newtheorem{remark}{Remark}[section]

\title{\textbf{LLM-Prior: A Framework for Knowledge-Driven Prior Elicitation and Aggregation}}
\author{Yongchao Huang \footnote{Author email: yongchao.huang@abdn.ac.uk}}
\date{4 August 2025}

\begin{document}

\maketitle

\begin{abstract}
The specification of prior distributions is fundamental in Bayesian inference, yet it remains a significant bottleneck. The prior elicitation process is often a manual, subjective, and unscalable task. We propose a novel framework which leverages Large Language Models (LLMs) to automate and scale this process. We introduce \texttt{LLMPrior}, a principled operator that translates rich, unstructured contexts such as natural language descriptions, data or figures into valid, tractable probability distributions. We formalize this operator by architecturally coupling an LLM with an explicit, tractable generative model, such as a Gaussian Mixture Model (forming a LLM based Mixture Density Network), ensuring the resulting prior satisfies essential mathematical properties. We further extend this framework to multi-agent systems where Logarithmic Opinion Pooling is employed to aggregate prior distributions induced by decentralized knowledge. We present the federated prior aggregation algorithm, \texttt{Fed-LLMPrior}, for aggregating distributed, context-dependent priors in a manner robust to agent heterogeneity. This work provides the foundation for a new class of tools that can potentially lower the barrier to entry for sophisticated Bayesian modeling.
\end{abstract}

\section{Introduction}
\label{sec:introduction}

Bayesian inference provides a powerful and coherent paradigm for reasoning under uncertainty, formalized by Bayes' theorem: $p(\theta|D) \propto p(D|\theta)p(\theta)$, where $p(\theta)$ is the prior distribution over parameters $\theta$. The prior is a fundamental component of any Bayesian analysis, encoding existing knowledge and regularizing the model, especially when data $D$ is scarce. However, the process of specifying this prior, known as prior elicitation, is a well-known challenge \cite{kadane_interactive_1980,ohagan2006uncertain}. Traditionally, it involves a labor-intensive and subjective process of the modellor interviewing domain experts to translate their qualitative beliefs into a quantitative, valid probability density function (PDF). This process is difficult to scale, reproduce, and is susceptible to cognitive biases.

The difficulty of prior elicitation represents a significant barrier to the widespread adoption of Bayesian methods. This paper introduces a framework to address this challenge by leveraging the remarkable capabilities of modern Large Language Models (LLMs). We reconceptualize the LLM not merely as a text generator, but as a 'knowledge compiler' - a function that can translate high-level, unstructured context into formal, structured mathematical objects.

This work lays the theoretical and architectural groundwork for a new generation of Bayesian modeling tools, aiming to make the specification of informative priors more systematic, scalable, and accessible. Two primary contributions are made: first, we formalize the concept of an \texttt{LLMPrior} operator, $\mathcal{L}$, which maps a flexible context $C$ (e.g. text, data, figures) to a valid prior distribution $p(z|C)$. The key challenge lies in ensuring the output of the operator is a valid PDF. We solve this by proposing a 'separation of concerns' architecture where the LLM is responsible for semantic interpretation, generating the parameters for a separate, explicit generative model (e.g. a Mixture Density Network) that guarantees validity by construction. Second, we extend this framework to a multi-agent setting, a common scenario in fields such as \textit{federated learning} or \textit{expert panels} where knowledge is distributed. We consider a system of $N$ agents, each with a local context $C_i$, who independently generate a prior $p_i(z) = \mathcal{L}(C_i)$. We address the problem of aggregating these $N$ priors into a single, coherent group prior, $p_{\text{agg}}(z)$. Drawing on opinion pooling theory, we argue for the principled use of the \textit{Logarithmic Opinion Pool} over simpler heuristics, and propose a concrete algorithm, \texttt{Fed-LLMPrior}, for its implementation in a centralized system.

\section{Related Work}
\label{sec:related_work}

This work sits at the intersection of Bayesian statistics, deep generative modeling, and natural language processing. As such, we briefly refer to related work in these domains.

\paragraph{Prior elicitation.} The challenge of eliciting priors from experts has a long history in Bayesian statistics \cite{kadane_interactive_1980, ohagan2006uncertain}. Most methods rely on structured questionnaires and expert interviews, which are difficult to scale. Our approach seeks to automate this process by treating the vast corpus of text an LLM is trained on as a proxy for expert knowledge, accessible via natural language prompts.

\paragraph{Deep generative models.} Our proposed architecture for the \texttt{LLMPrior} relies on deep generative models that afford a tractable density. Mixture Density Networks (MDNs) \cite{bishop1994mixture}, which use a neural network to output the parameters of a mixture model, are a prime candidate due to their simplicity and universal approximation properties \cite{goodfellow2016deep}. Normalizing Flows (NFs) \cite{rezende2015variational,dinh2016density} offer another powerful alternative, constructing complex densities via invertible transformations of a simple base distribution. Our framework leverages these existing architectures in a novel way: instead of training them on data samples, we use an LLM to generate their parameters directly from a high-level context. A comparison of these architectures can be found in Table.\ref{tab:arch_comparison} in Appendix.\ref{app:comparison_tables}.

\paragraph{Opinion pooling.} The aggregation of expert beliefs, formalized as probability distributions, is the subject of opinion pooling theory \cite{stone1961opinion,genest1986combining,cooke_experts_1991}. The two dominant methods are the Linear Opinion Pool (LOP) \cite{degroot_optimal_1991} and the Logarithmic Opinion Pool (LogP) \cite{bordley_multiplicative_1982,heskes1998selecting}. In general, LOP averages and compromises, while LogP multiplies and seeks a consensus that can be more extreme than any individual belief. LOP takes
a weighted linear average of the distributions and is simpler \cite{heskes1998selecting,carvalho_consensual_2013}. The LogP, which forms a geometric mean of densities, is 'externally Bayesian' \footnote{The logarithmic opinion pool is 'externally Bayesian', i.e. can be derived from joint probabilities using Bayes' rule \cite{bordley_multiplicative_1982,heskes1998selecting}. In group decision-making, an 'externally Bayesian' group \cite{madansky1964externally} is one where the group's collective decisions appear to an outsider as if they were made by a single, rational Bayesian agent, meaning that the group's decisions, when viewed as a whole, follow the principles of Bayes' theorem regarding updating beliefs based on new evidence. A key property is that the final outcome remains the same regardless of the order in which the group pools information and updates beliefs.} and has a strong information-theoretic justification as the distribution minimizing the KL-divergence to the individual beliefs \cite{heskes1998selecting}. Our framework adopts the LogP as the principled choice for prior aggregation due to its externally Bayesian, prior to posterior coherent,  data independence, and log-concavity preservation properties \cite{genest1986combining,carvalho2023bayesian}. We apply this classical theory to modern distributed systems with federated learning \cite{makhija_bayesian_2024}. 

\paragraph{LLM assisted modeling.} There is a growing line of work using LLMs to assist in scientific and modeling tasks. This includes work on LLM-driven symbolic regression, where models generate mathematical expressions to fit data \cite{shojaee_llm-sr_2025}, which demonstrates the ability of LLMs to produce structured, mathematical outputs. Most relevant to our work is the concurrent development of the 'Large Language Bayes' (LLB) framework by Domke \cite{domke2025large}. LLB uses an LLM to generate a distribution over full probabilistic programs, $p(m|t)$, from a text description $t$. It then performs Bayesian model averaging over these LLM-generated models. LLB provides compelling empirical evidence that LLMs can automate parts of the Bayesian workflow. Our work is complementary: while LLB addresses model uncertainty for a single agent, our framework focuses on belief aggregation across multiple agents with diverse contexts.

\section{Methodology}
\label{sec:methodology}

\subsection{The \texttt{LLMPrior} operator}

We formally define the \texttt{LLMPrior} as an operator $\mathcal{L}$ that maps a context $C$ from a flexible space of contexts $\mathcal{C}$ to a probability distribution over a variable $z \in \mathbb{R}^d$. The operator is parameterized by the weights of an LLM, $\theta_{LLM}$:
\begin{equation}
    p(z | C) = \mathcal{L}_{\theta_{LLM}}(C)
\end{equation}

For $p(z|C)$ to be a valid and useful prior, it must satisfy following properties:
\begin{enumerate}
    \item \textbf{Non-negativity (essential):} $p(z|C) \ge 0$ for all $z$.
    \item \textbf{Normalization (essential):} $\int p(z|C) dz = 1$.
    \item \textbf{Expressiveness (essential)}: $p(z|C)$ must be flexible enough to represent a wide range of beliefs (e.g. unimodal, multimodal, skewed). 
    \item \textbf{Differentiability (desirable):} we prefer $p(z|C)$ to be differentiable \textit{w.r.t.} its parameters, to be compatible with gradient-based inference methods (e.g. HMC, VI).
    \item \textbf{Tractability (desirable):} both density evaluation and sampling from $p(z|C)$ be computationally tractable.
\end{enumerate}

A naive approach of regressing from $C$ to the values of $p(z)$, e.g. \textit{text-to-text regression} \cite{raffel_exploring_2020,akhauri_performance_2025}, would fail to satisfy these properties. Instead, we ask the LLM to generate the parameters $\phi$ of a separate, explicit generative model $f(z; \phi)$ which is guaranteed to be a valid PDF by construction.

\begin{definition}[\textit{MDN-LLM architecture}]
    We propose implementing $\mathcal{L}$ using a Mixture Density Network structure. The LLM's role is to map the context $C$ to the parameters $\phi = \{\{\alpha_k\}, \{\mu_k\}, \{L_k\}\}_{k=1}^K$ of a $K$-component Gaussian Mixture Model (GMM), i.e.
    \begin{equation}
        \phi = \text{LLM}(C)
    \end{equation}
    The final prior is the GMM constructed from these parameters:
    \begin{equation}
        p(z|C) = \sum_{k=1}^K \alpha_k \mathcal{N}(z | \mu_k, \Sigma_k)
    \end{equation}
To ensure validity, constraints are enforced on the LLM's raw outputs:
    \begin{itemize}
        \item \textbf{Mixture weights:} the raw outputs of the LLM are passed through a softmax function to ensure $\alpha_k \ge 0$ and $\sum_k \alpha_k = 1$.
        \item \textbf{Means}: linear activation is used as means $\mu_k$ are unconstrained.
        \item \textbf{Covariance matrices:} the LLM also outputs the Cholesky factors $L_k$, with positive diagonal elements. The covariance matrices are then constructed as $\Sigma_k = L_k L_k^T$, which guarantees they are symmetric and positive semi-definite.
    \end{itemize}
\end{definition}
This MDN-LLM architecture satisfies all the 5 properties by construction. GMMs are universal approximators of densities \cite{li_mixture_1999,norets_approximation_2010}, they provide sufficient flexibility to capture complex beliefs (expressiveness).
A GMM is a valid, differentiable PDF. 
Both density evaluation and sampling are tractable \cite{blei_variational_2017,gatmiry_learning_2025}.  
The GMM has tractable normalization constant: the integral of GMM, known as the partition function, is $\int p(z|C) dz = \int \sum \alpha_k \mathcal{N}(z|\mu_k, \Sigma_k) dz = \sum \alpha_k \int \mathcal{N}(z|\mu_k, \Sigma_k) dz = \sum \alpha_k = 1$.

For the LLM, we would expect it to have contextual sensitivity, i.e. the mapping from context $C$ to parameters $\phi$ must be sensitive and meaningful: similar contexts should produce similar distributions, and semantically different contexts should produce appropriately different distributions. This property depends on the design and training of the underlying LLM.

\subsection{Aggregation of distributed priors}

We now consider a system of $N$ agents, where each agent $i$ has a local context $C_i$ and generates a local prior $p_i(z) = \mathcal{L}(C_i)$. The goal is to aggregate the set of priors $\{p_1(z), \dots, p_N(z)\}$ into a single group prior $p_{\text{agg}}(z)$ for downstream Bayesian modelling.

\begin{definition}[\textit{Logarithmic opinion pool (LogP)}]
The Logarithmic Opinion Pool constructs the aggregate distribution as a weighted geometric mean of the individual densities:
\begin{equation}
    p_{\text{agg}}(z) \propto \prod_{i=1}^N p_i(z)^{w_i}
\end{equation}
where the weights $w_i \ge 0$ and $\sum_i w_i = 1$ represent the relative expertise or reliability of each agent.
\end{definition}

We advocate for the LogP over the simpler Linear Opinion Pool ($p_{\text{agg}}(z) = \sum_i w_i p_i(z)$) for several reasons \footnote{A comparison of distribution aggregation strategies can be found in Table.\ref{tab:agg_comparison} in Appendix.\ref{app:comparison_tables}.}. Most importantly, the LogP is \textit{externally Bayesian}, meaning it commutes with Bayesian updating:

\begin{remark}[\textit{External Bayesianity}, Carvalho et al. \cite{carvalho2023bayesian,genest1984characterization,genest1984aggregating}]
Suppose we have a collection of prior densities $\{p_i(z)\}$ and observe data $x$ with a common likelihood $l(x \mid z)$. Then, updating each $p_i(z)$ individually yields posteriors
\[
p_i(z \mid x) \propto l(x \mid z)\, p_i(z).
\]
External Bayesianity holds if aggregating these posteriors results in the same distribution as first aggregating the priors into a single prior $p_{\text{agg}}(z)$, and then updating it using the likelihood:
\[
p(z \mid x) \propto l(x \mid z)\, p_{\text{agg}}(z).
\]
\end{remark}

Genest et al.\cite{genest1984aggregating} showed that LogP is the only aggregation (pooling) operator that enjoys external Bayesianity.
A group using a LogP will reach the same posterior belief regardless of whether they pool their priors before observing new data or after. This property is important for a coherent, rational group agent. The LogP also has a strong information-theoretic justification, as it is the distribution that minimizes the weighted sum of KL-divergences to the individual priors \cite{heskes1998selecting}. Further, Carvalho et al. \cite{carvalho2023bayesian} proved that, logarithmic pooling is the only aggregation method to universally preserve log-concavity \footnote{The log-concavity property makes sampling easier for some sampling algorithms such as slice sampling \cite{neal_slice_2003}, as they rely on log-concavity of the target distribution \cite{carvalho2023bayesian}.}:

\begin{theorem}[\textit{Log-concavity preservation via logarithmic pooling, Carvalho et al. \cite{carvalho2023bayesian}}]
\label{thm:log_concavity}
Let $\{p_i(z)\}_{i=1}^K$ be a collection of log-concave distributions and assume mild regularity conditions hold. Then the aggregated distribution $p_{\text{agg}}(z)$ obtained via logarithmic pooling is also log-concave. Furthermore, \emph{logarithmic pooling} is the \emph{only} pooling operator that \emph{guarantees} log-concavity of the aggregated distribution whenever all input distributions $p_i(z)$ are log-concave.
\end{theorem}
\begin{proof}
See Carvalho et al. \cite{carvalho2023bayesian}.
\end{proof}

\subsection{\texttt{Fed-LLMPrior}: aggregating priors in federated learning}

We synthesize these components into a concrete algorithm for a centralized, federated setting, as shown in Algo.\ref{alg:fed-llmprior}.

\begin{algorithm}[H]
   \caption{\texttt{Fed-LLMPrior} Framework}
   \label{alg:fed-llmprior}
\begin{algorithmic}
   \STATE {\bfseries System:} $N$ agents, 1 central server.
   \STATE {\bfseries Architecture:} \texttt{LLMPrior} implemented as an MDN-LLM with $K$ components.
   \STATE
   \STATE {\bfseries Server Initialization:}
   \STATE Define problem space for $z \in \mathbb{R}^d$, number of components $K$.
   \STATE Broadcast task specification to all agents.
   \STATE
   \STATE {\bfseries Local Prior Generation (in parallel for each agent $i=1, \dots, N$):}
   \STATE Agent $i$ acquires local context $C_i$.
   \STATE Agent $i$ queries its LLM to generate raw parameters: $\phi_i^{\text{raw}} = \text{LLM}(\text{prompt}(C_i))$.
   \STATE Agent $i$ validates and structures parameters $\phi_i = \{\{\alpha_{ik}\}, \{\mu_{ik}\}, \{\Sigma_{ik}\}\}_{k=1}^K$.
   \STATE Agent $i$ constructs its local prior: $p_i(z) = \sum_{k=1}^K \alpha_{ik} \mathcal{N}(z | \mu_{ik}, \Sigma_{ik})$.
   \STATE Agent $i$ sends its structured parameters $\phi_i$ to the server.
   \STATE
   \STATE {\bfseries Aggregation (Server-side):}
   \STATE Server collects parameter sets $\{\phi_1, \dots, \phi_N\}$.
   \STATE Server computes the unnormalized aggregated prior via LogP:
   $$ p_{\text{agg}}^{\text{unnorm}}(z) = \prod_{i=1}^N \left( \sum_{k=1}^K \alpha_{ik} \mathcal{N}(z | \mu_{ik}, \Sigma_{ik}) \right)^{w_i} $$
   \STATE The result is a product of GMMs, which is an intractable mixture of $K^N$ components.
   \STATE Server approximates $p_{\text{agg}}^{\text{unnorm}}(z)$ with a tractable GMM, $p_{\text{agg}}(z)$, using variational inference or moment matching.
   \STATE
   \STATE {\bfseries Output:} the aggregated prior $p_{\text{agg}}(z)$ is used for downstream Bayesian modeling.
\end{algorithmic}
\end{algorithm}

This algorithm provides a principled, end-to-end pipeline for constructing an aggregated prior $p_{\text{agg}}(z)$, considering information from all distributed agents and preserving privacy. It avoids the naive heuristic of averaging the GMM parameters, which is theoretically unsound under heterogeneous contexts (analogous to the problem of unbalanced and non-IID data distributions frequently encountered in federated learning \cite{mcmahan2017communication}). By aggregating at the level of distributions via LogP, the framework is robust to the diversity of agent knowledge.

\section{Experiments}
\label{sec:experiments}
To provide an empirical proof-of-concept for the \texttt{LLMPrior} framework, we conduct a set of foundational experiments. These experiments are designed to test the core hypothesis: an LLM can translate natural language context into a valid, meaningful prior distribution, and that this prior behaves correctly within a standard Bayesian workflow.

\subsection{Experimental Design}

\paragraph{Task 1: Eliciting a prior for a binomial proportion.}
The goal is to estimate the bias $\theta \in [0, 1]$ of a coin, a classic problem in Bayesian inference. The conjugate prior for the Binomial likelihood is the Beta distribution, $p(\theta|a,b) \propto \theta^{a-1}(1-\theta)^{b-1}$. 
The task for the \texttt{LLMPrior} is to generate the hyperparameters $(a,b)$ from natural language contexts $C$, i.e. $p(\theta|a,b)=Beta(\theta | \mathcal{L}(C))$. Since the Beta distribution is the conjugate prior for the Binomial likelihood, the posterior distribution is also a Beta distribution. Given a prior $\text{Beta}(a, b)$ and observing $k$ successes (heads) in $n$ trials (flips), the posterior parameters are updated analytically to $a' = a+k$ and $b' = b + (n-k)$.

We designed three distinct contexts to represent different belief states:
\begin{itemize}
    \item \textit{Uninformative:} \texttt{"I have no information about the coin. Any bias is equally likely."} (Expected: $a \approx 1, b \approx 1$)
    \item \textit{Fair:} \texttt{"The coin is believed to be fair. I am quite confident in this belief."} (Expected: $a=b \gg 1$)
    \item \textit{Biased:} \texttt{"The coin is strongly biased towards heads. I am very certain it is not fair and lands heads most of the time."} (Expected: $a > b$, with $a$ significantly larger than $b$)
\end{itemize}

The prompt used is (with $\texttt{context\_text}$ replaced by each of the above 3 contexts):
\begin{verbatim}
"""
You are an expert statistician. Your task is to elicit a prior 
distribution for the bias of a coin, theta, which lies in the range 
[0, 1]. The appropriate prior is a Beta(a, b) distribution. Based on 
the following context, determine the most appropriate hyperparameters 
'a' and 'b'. Your response MUST be a valid JSON object with two keys: 
"a" and "b", with positive numerical values. Do not include any other 
text, explanations, or markdown formatting.

CONTEXT: "{context_text}"
"""
\end{verbatim}

For implementation, a \textit{Gemini} model \cite{comanici2025gemini25pushingfrontier} (version 2.5) was prompted to output a JSON object with keys `a` and `b`. For validation, we perform a Bayesian update after observing hypothetical data of 8 heads and 2 tails (an observed heads proportion of 0.8) and analyze the resulting posterior.

\paragraph{Task 2: Aggregation of conflicting priors.}
To validate the aggregation mechanism, we extend the coin-flipping task to a two-agent scenario. Agent 1 will receive a context suggesting a slight bias towards heads: 
\texttt{"I have a weak suspicion that the coin might be slightly biased towards heads, but I am not very certain."},
while Agent 2 receives a context suggesting a slight bias towards tails: 
\texttt{"From what I recall, the coin seems to have a small tendency to land on tails more often, though my confidence is low."}, the remaining prompt used is the same as Task 1.
Each agent generates a Beta prior, $p_1(\theta|a_1, b_1)=\text{Beta}(\theta | \mathcal{L}(C_1))$ and $p_2(\theta|a_2, b_2)=\text{Beta}(\theta | \mathcal{L}(C_2))$. The server will then aggregate them using a LogP with equal weights ($w_1=w_2=0.5$). The aggregated prior is analytically tractable: the LogP of two Beta distributions results in a new Beta distribution (see Appendix.\ref{app:aggregation_of_two_Betas_derivation} detailed derivation), $p_{\text{agg}}(\theta) \sim \text{Beta}(0.5(a_1+a_2), 0.5(b_1+b_2))$. We will verify that the aggregation of two opposing, weakly-held beliefs results in a more uncertain, centered prior.

\paragraph{Task 3: Eliciting a bimodal prior with MDN-LLM.}
To demonstrate the framework's ability to capture more complex, multi-modal beliefs, our third task focuses on eliciting a Gaussian Mixture Model (GMM) prior. We use the classic problem of modeling the waiting times between eruptions of the Old Faithful geyser in Yellowstone National Park, which is known to exhibit bimodal behavior \cite{azzalini_look_1990}. The goal is to generate a two-component GMM prior for the waiting time $\theta \in \mathbb{R}^+$.

Following natural language context (prompt) is provided to the LLM to describe this bimodality:
\begin{verbatim}
"""
You are an expert statistician. Your task is to elicit a prior distribution 
for the waiting time of a geyser eruption. The appropriate prior is a 2-component 
Gaussian Mixture Model (GMM). Based on the following context, determine the most 
appropriate parameters for this GMM. The parameters are: the mixture weights (a list 
of 2 floats that sum to 1), the means (a list of 2 floats), and the standard 
deviations (a list of 2 positive floats).

Your response MUST be a valid JSON object with three keys: "weights", "means", 
and "std_devs". Do not include any other text, explanations, or markdown formatting.

CONTEXT: "The waiting time between eruptions of the Old Faithful geyser is 
known to be variable. Historical data suggests there are two distinct 
patterns: a shorter waiting time and a longer waiting time. The shorter 
waits seem to be centered around 55 minutes, while the longer ones are 
closer to 80 minutes. The shorter waits appear to be slightly less 
common than the longer waits. Both patterns have a similar, moderate 
level of variability."
"""
\end{verbatim}

The LLM will be prompted to act as the core of an MDN, generating the parameters for a 2-component GMM: the mixture weights ($\alpha_1, \alpha_2$), the means ($\mu_1, \mu_2$), and the standard deviations ($\sigma_1, \sigma_2$). The expected output is a JSON object containing these parameters. Validation will consist of plotting the generated GMM distribution to visually confirm that it correctly represents the bimodal structure described in the context.

\subsection{Task 1 results: eliciting a prior for a binomial proportion}
The results from the first experiment strongly support the feasibility of our framework. The LLM successfully interpreted the semantic content of each context and generated appropriate hyperparameters for the Beta prior. The subsequent Bayesian updates demonstrate that these machine-generated priors behave coherently and predictably.

\begin{figure}[H]
    \centering
    \includegraphics[width=0.6\textwidth]{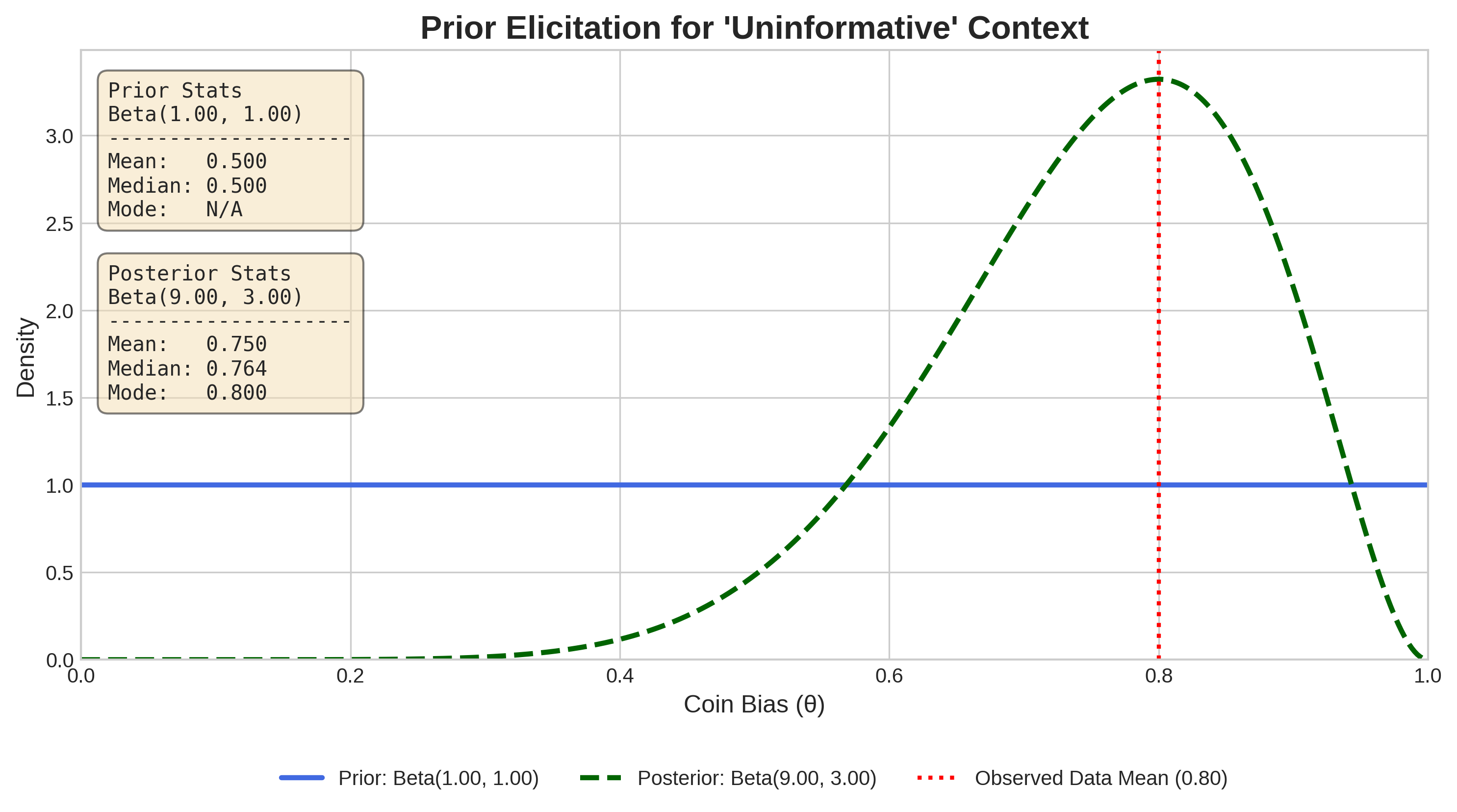}
    \caption{Result for the 'Uninformative' context. The LLM correctly generates a uniform Beta(1,1) prior. The posterior is dominated by the data, with its mode at 0.8.}
    \label{fig:uninformative}
\end{figure}

\paragraph{Uninformative context.} As shown in Fig.\ref{fig:uninformative}, when presented with the context \textit{'I have no information about the coin. Any bias is equally likely.'}, the LLM generated parameters $a=1.00$ and $b=1.00$. This corresponds to a Beta(1,1) distribution, which is a uniform distribution over the interval [0, 1]. This is the correct and standard choice for an uninformative prior in this setting, demonstrating the LLM's ability to map the concept of 'no information' to the correct mathematical form. After observing 8 heads and 2 tails, the posterior becomes Beta(9.00, 3.00). The posterior's mode is 0.800, exactly matching the empirical proportion of heads from the data. This illustrates a general principle: when the prior is uninformative, the posterior belief is shaped entirely by the observed data (i.e. dominated by likelihood).

\begin{figure}[H]
    \centering
    \includegraphics[width=0.6\textwidth]{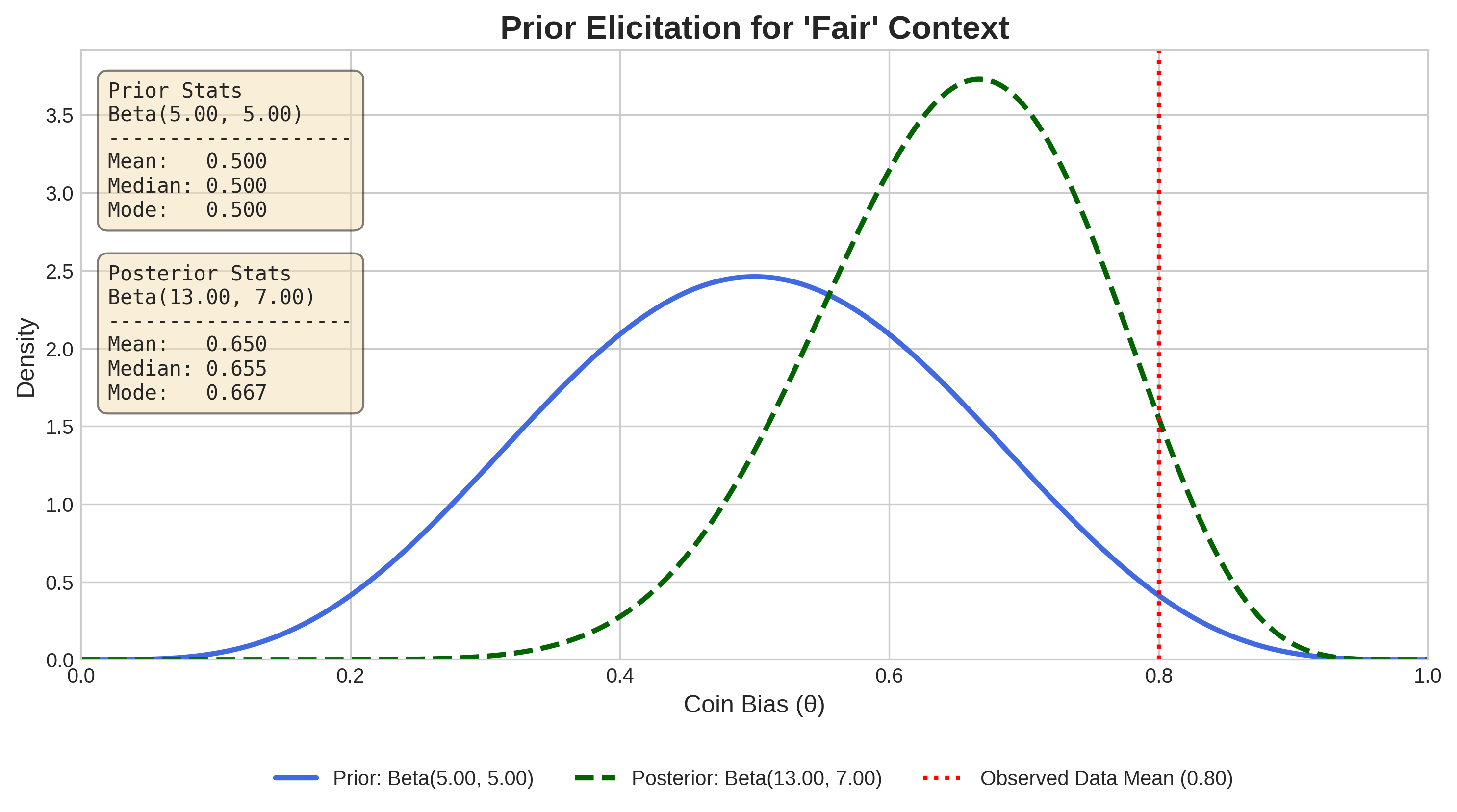}
    \caption{Result for the 'Fair' context. The prior is a Beta(5.00, 5.00), centered at 0.5. The posterior is pulled towards the data, but the prior's influence keeps the posterior mode (0.667) well below the data's mean (0.8).}
    \label{fig:fair}
\end{figure}

\paragraph{Fair context.} For the context \textit{'The coin is believed to be fair. I am quite confident in this belief.'}, the LLM produced a Beta(5.00, 5.00) prior (Fig.\ref{fig:fair}). This is a symmetric, unimodal distribution centered at 0.5, which correctly captures the belief in a fair coin. The parameters $a=b=5$ represent a moderate degree of confidence in this belief (equivalent to having previously observed 4 heads and 4 tails). The resulting posterior is Beta(13.00, 7.00). We observe the interplay between prior and data: the posterior mode shifts to 0.667, moving significantly towards the data's mean of 0.8, but the prior's 'gravity' prevents the posterior from fully converging to the data. This demonstrates the regularizing effect of an informative prior.

\begin{figure}[H]
    \centering
    \includegraphics[width=0.6\textwidth]{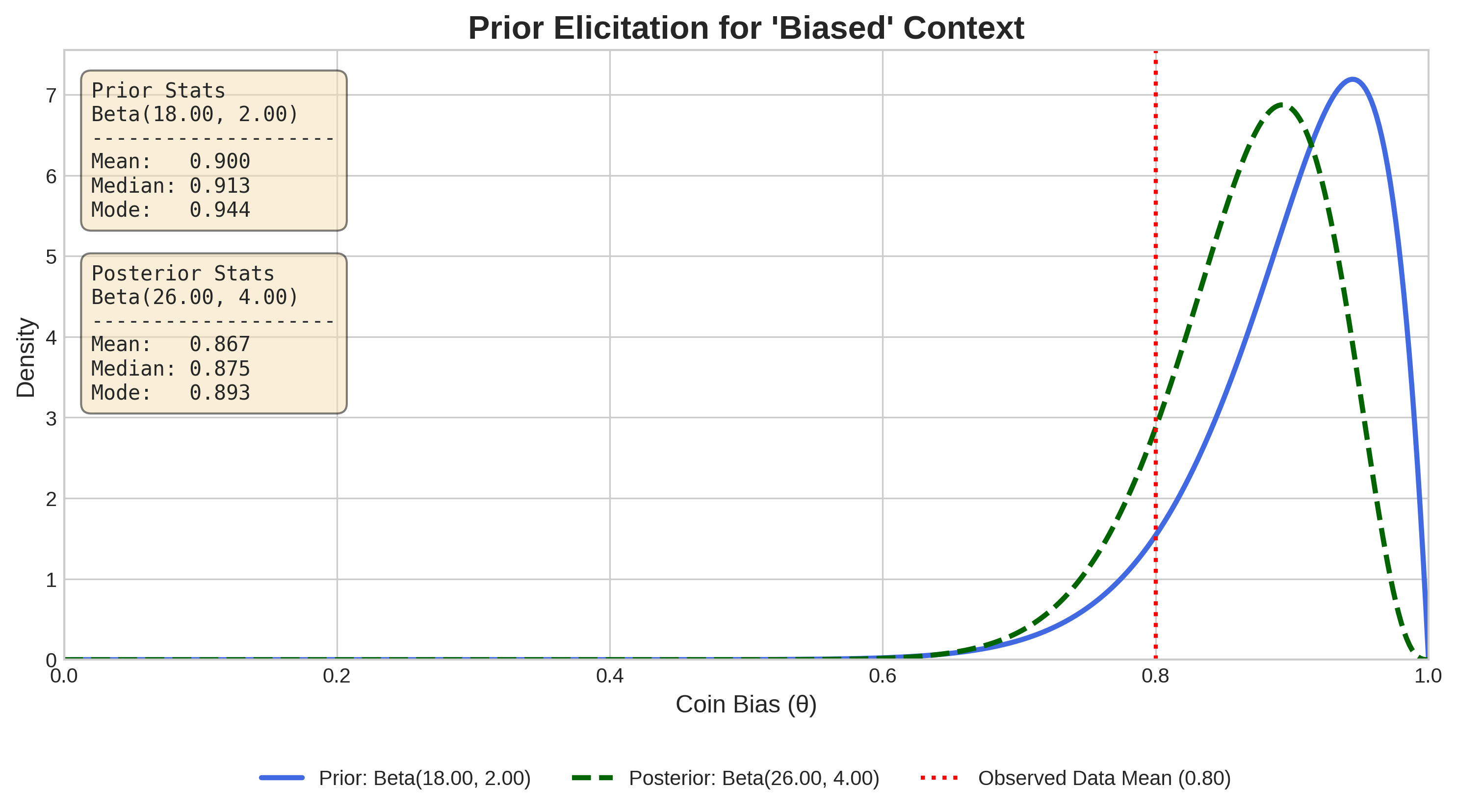}
    \caption{Result for the 'Biased' context. The LLM generates a strong prior, Beta(18.00, 2.00), reflecting the belief in a bias towards heads. The posterior, Beta(26.00, 4.00), is only slightly shifted by the data, as the prior and data are in agreement.}
    \label{fig:biased}
\end{figure}

\paragraph{Biased context.} Given the context \textit{'The coin is strongly biased towards heads. I am very certain it is not fair and lands heads most of the time.'}, the LLM generated a Beta(18.00, 2.00) prior (Fig.\ref{fig:biased}). This is a strong prior with a mode at 0.944, accurately reflecting the specified belief. The large parameters indicate high confidence. In this case, the prior belief is highly consistent with the observed data (8 heads, 2 tails). The resulting posterior, Beta(26.00, 4.00), has a mode at 0.893. The posterior is very similar to the prior, as the new data serves to confirm the existing strong belief rather than challenge it. The small shift in the mode reflects a minor update based on the empirical evidence from data (i.e. contribution from likelihood).

Collectively, these results provide a successful proof-of-concept. They show that the \texttt{LLMPrior} operator can correctly interpret nuanced natural language, generate corresponding valid prior distributions, and that these priors function as expected within a Bayesian update, exhibiting weak or strong influence on the posterior based on their initial strength and their agreement with the data.

\subsection{Task 2 results: aggregating conflicting priors}
To test the aggregation mechanism, we executed the second task, where two agents with conflicting, weakly-held beliefs generated priors that were then aggregated using LogP. The results, shown in Fig.\ref{fig:aggregation}, demonstrate that the framework can successfully synthesize opposing views into a single, coherent consensus distribution.

\begin{figure}[H]
    \centering
    \includegraphics[width=0.6\textwidth]{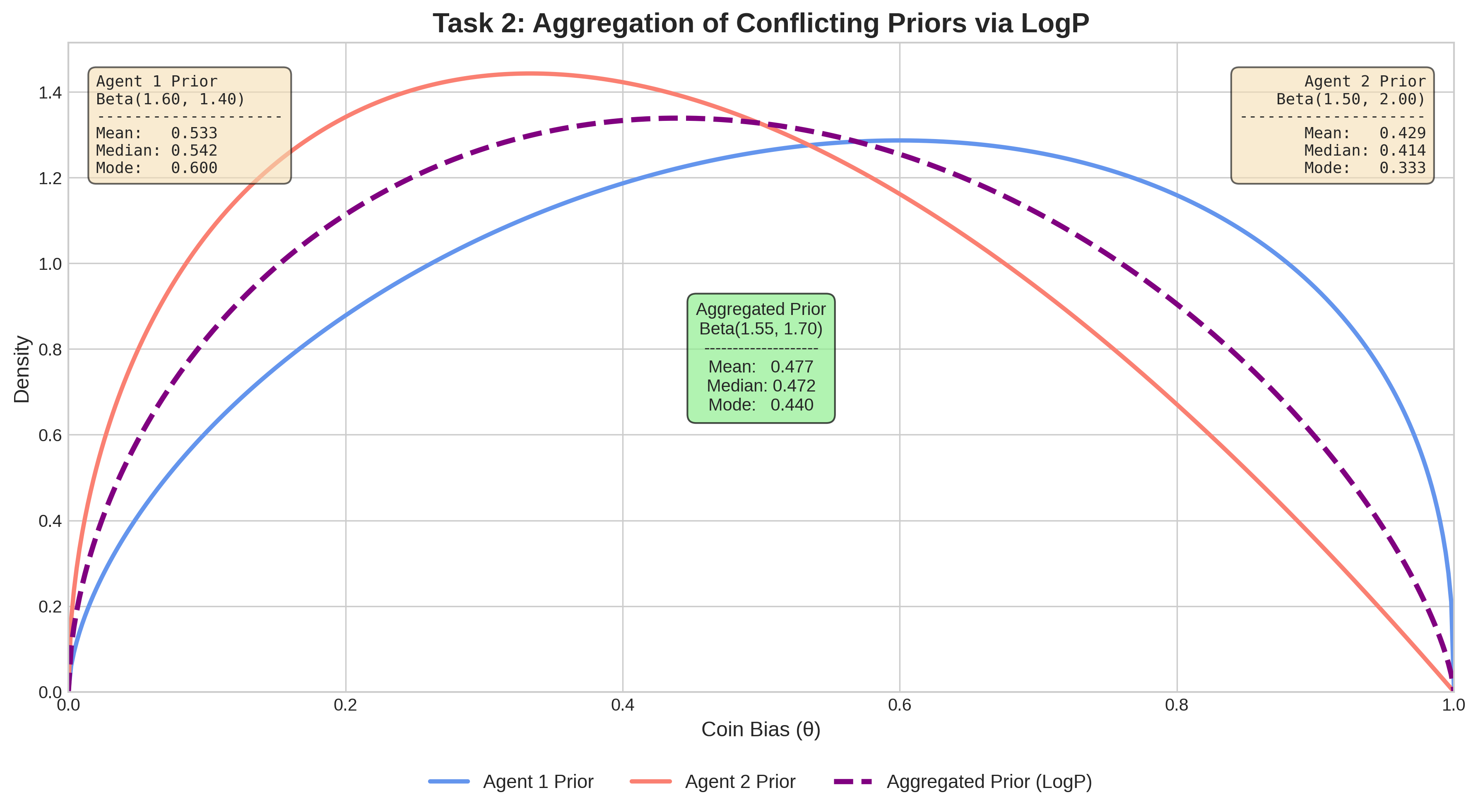}
    \caption{Result for Task 2. Agent 1 (blue) and Agent 2 (brown) have conflicting, weak priors. The LogP aggregation (purple, dashed) produces a single, more centered prior that reflects the combined uncertainty of both agents.}
    \label{fig:aggregation}
\end{figure}

\paragraph{Analysis of Aggregation.}
Agent 1, given the context of a '\textit{weak suspicion}' of a heads bias, produced a Beta(1.60, 1.40) prior. This is a unimodal distribution with a mode at 0.600 and a mean of 0.533, correctly capturing a slight belief in favor of heads. Agent 2, with a context suggesting a '\textit{small tendency to land on tails}' with low confidence, produced a Beta(1.50, 2.00) prior. This is also a unimodal distribution, but it is skewed towards tails with a mode at 0.333 and a mean of 0.429, accurately reflecting its context.

The Logarithmic Opinion Pool, with equal weights, aggregated these two distributions into a final Beta(1.55, 1.70) prior. This aggregated prior has a mean of 0.477 and a mode of 0.440. This outcome is highly significant for two reasons. First, the aggregation of two opposing beliefs did not result in an incoherent or overly broad distribution, but rather a single, rational consensus belief that is centered between the two initial priors. Second, the resulting prior is slightly broader (more uncertain) than the two individual priors, which correctly reflects the logical outcome of combining two weak, conflicting pieces of evidence: the rational consensus is a state of compromise that acknowledges the uncertainty from both sources. This experiment provides a strong proof-of-concept for the LogP's ability to robustly and coherently aggregate heterogeneous priors within the \texttt{LLMPrior} framework.

\subsection{Task 3 results: eliciting a bimodal GMM prior}
The third experiment tested the framework's ability to move beyond simple unimodal distributions and elicit a more complex, bimodal prior using the MDN-LLM architecture. The results, shown in Fig.\ref{fig:gmm_prior}, are a compelling demonstration of the operator's capacity to translate nuanced, qualitative descriptions into a correctly structured Gaussian Mixture Model.

\begin{figure}[H]
    \centering
    \includegraphics[width=0.6\textwidth]{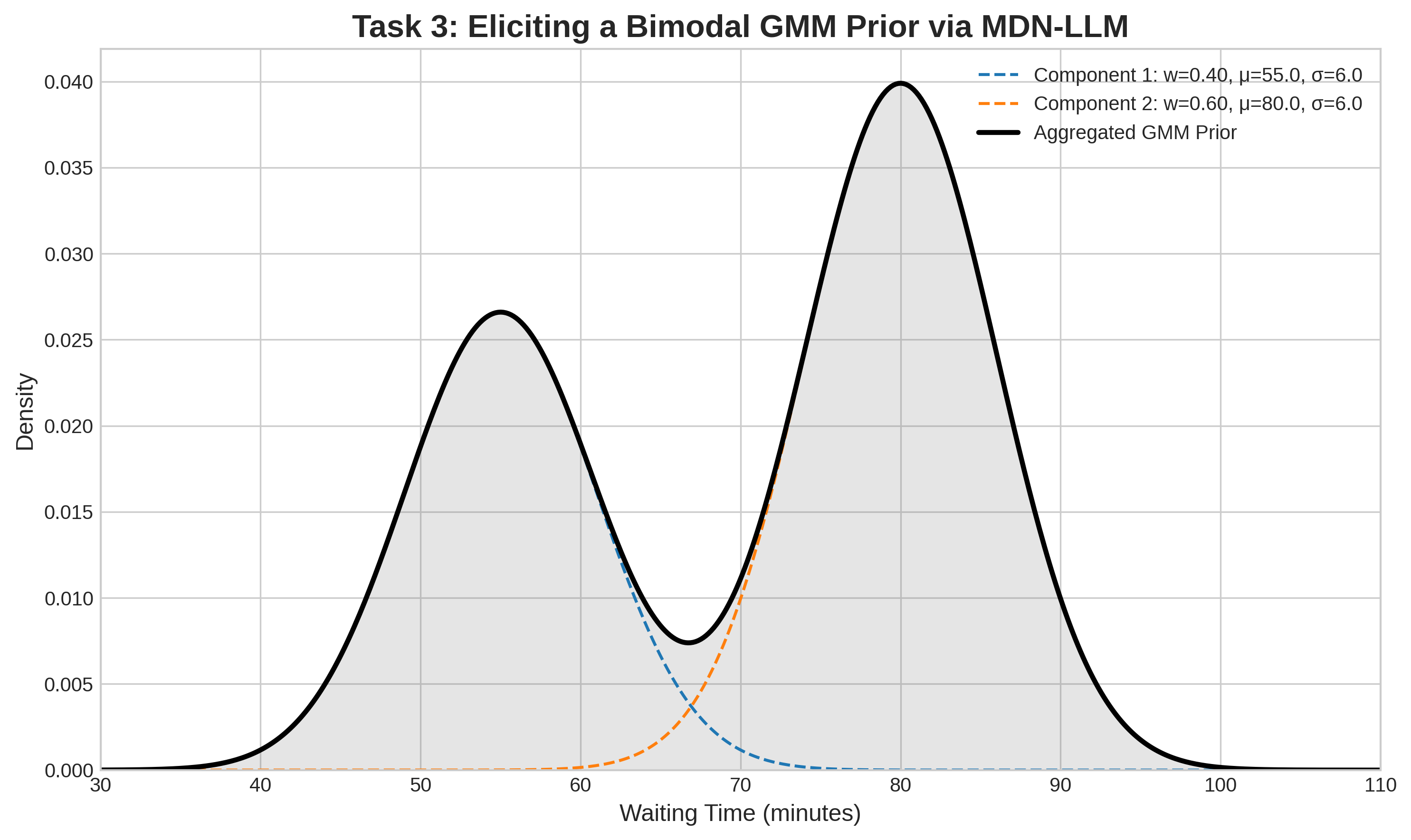}
    \caption{Result for Task 3. The LLM successfully generates a bimodal GMM prior that accurately reflects the natural language description of the Old Faithful geyser data, with correctly placed means and appropriate weighting.}
    \label{fig:gmm_prior}
\end{figure}

\paragraph{Analysis of GMM elicitation.}
Presented with a detailed context describing the bimodal nature of the Old Faithful geyser's eruption waiting times, the LLM was prompted to generate the parameters for a two-component GMM. The model's output was remarkably accurate:
\begin{itemize}
    \item \textbf{Means:} the context mentioned "shorter waits... centered around 55 minutes" and "longer ones are closer to 80 minutes." The LLM generated means of $\mu_1=55.0$ and $\mu_2=80.0$, perfectly matching the specified locations of the modes.
    \item \textbf{Mixture weights:} the context stated that "shorter waits appear to be slightly less common than the longer waits." The LLM interpreted this by assigning a weight of $\alpha_1=0.4$ to the first component and $\alpha_2=0.6$ to the second, correctly capturing the described asymmetry in the mixture.
    \item \textbf{Standard deviations:} in response to the phrase "Both patterns have a similar, moderate level of variability," the LLM assigned an equal standard deviation of $\sigma_1=\sigma_2=6.0$ to both components, a reasonable quantitative interpretation of the context.
\end{itemize}

As visualized in Fig.\ref{fig:gmm_prior}, these parameters combine to form a bimodal distribution that is an excellent representation of the prior belief described in the text. This result successfully demonstrates that the \texttt{LLMPrior} framework, and specifically the MDN-LLM architecture, is not limited to simple conjugate families but can effectively elicit more complex, multi-modal priors, significantly broadening its applicability to real-world modeling problems.

\section{Discussion}
\label{sec:discussion}
The successful execution of our experiments provides a strong proof-of-concept, yet it also illuminates some practical challenges and promising directions for future research.

\paragraph{Computational costs and scalability.}
The reliance on LLMs for prior generation incurs computational costs, but a more significant challenge for scalability lies in the aggregation step. As we move from simple Beta priors (Tasks 1 \& 2) to more expressive GMM priors (Task 3), the aggregation complexity becomes paramount. Our framework specifies using a LogP, but the product of $N$ GMMs with $K$ components each results in a new mixture with $K^N$ components. This combinatorial explosion makes exact aggregation intractable. The development of efficient and accurate approximation schemes, whether through variational methods, moment matching, or sophisticated sampling, is the most critical research direction for applying this framework to large-scale, multi-agent systems using complex priors.

\paragraph{Robustness to LLM artifacts.}
The framework's output quality relies on the LLM's behavior. Our experiments successfully used carefully engineered prompts that demanded a strict JSON output, which helps mitigate the risk of free-form "hallucinations" \cite{farquhar_detecting_2024,huang_survey_2025}. However, the risk of the LLM inheriting biases from its training data or generating nonsensical parameters remains. Future work could focus on creating robust validation layers to "sanity-check" generated parameters against physical or logical constraints (e.g. ensuring a variance is non-negative, or a mean falls within a plausible range).

\paragraph{Quantifying LLM-induced uncertainty.}
Our experiments generated point estimates of the prior parameters (e.g. Beta(5.00, 5.00)). This, however, ignores the uncertainty in the LLM's generation process itself. A more complete Bayesian treatment would capture this model uncertainty. A promising direction is to develop a hierarchical model where the LLM generates a distribution over prior parameters (a hyper-prior), for example by prompting it multiple times with slight variations \footnote{During our experiments, we found multiple calls, even with the same prompt, to the same LLM model can generate semantically consistent but slightly different feedbacks.} or a non-zero temperature setting. This would allow the final inference to marginalize over the LLM's own uncertainty, leading to more robust conclusions.

\paragraph{Sensitivity to input prompts.}
The success of our experiments relied on specific, un-ambiguous prompts. As noted in the methodology, the mapping from context $C$ to parameters $\phi$ must be sensitive and meaningful. While our results show this is achievable, it also highlights a potential vulnerability: the framework's performance is tied to the quality of prompt engineering. Future research could investigate methods for automatically optimizing prompts or developing conversational elicitation processes where the LLM can ask clarifying questions to resolve ambiguity before generating the final prior.

\paragraph{Integration with probabilistic programming languages.}
A key strength of this framework is its modularity, allowing the \texttt{LLMPrior} to be flexibly integrated into existing Probabilistic Programming Languages (PPLs). The operator is not intended to replace PPLs, but to augment them. A practitioner could define a model in a language such as \textit{Turing.jl} \cite{Fjelde2024Turing}, \textit{RxInfer.jl} \cite{Bagaev2023}, \textit{PyMC} \cite{salvatier_probabilistic_2015}, \textit{Pyro} \cite{bingham_pyro_2018}, \textit{PINTS} \cite{clerx_probabilistic_2019} or \textit{Stan} \cite{carpenter_stan_2017}, and call the \texttt{LLMPrior} operator within the model definition to specify a prior for a parameter. For example, a line in a probabilistic programme could read \footnote{The author thanks Albert Podusenko, whose post and insights on this topic inspired this line of thought.}: \texttt{$\theta \sim$ LLMPrior("context about $\theta$")}. In practice, this would issue a one-time call to the LLM to fetch the prior's parameters, which are then fixed. The PPL's inference engine would then proceed as usual, using the generated distribution for sampling. This "plug-and-play" capability significantly lowers the barrier to specifying complex, context-informed priors within established Bayesian modeling workflows.

\section{Conclusion}
\label{sec:conclusion}
We present a novel framework, \texttt{LLMPrior}, for automating and scaling the elicitation and aggregation of prior distributions using Large Language Models. We introduced a principled operator, $\mathcal{L}$, which translates unstructured context into a valid, tractable PDF by architecturally coupling an LLM with an explicit generative model. Our experiments successfully demonstrated this capability, showing that the framework can elicit both simple conjugate priors (Beta distributions) and more complex, multi-modal priors (Gaussian Mixture Models) that accurately reflect the semantic content of the provided context.

Furthermore, we extended this concept to multi-agent systems, advocating for the theoretically-grounded Logarithmic Opinion Pool to aggregate distributed beliefs. Our second experiment provided a proof-of-concept for this aggregation, showing that the LogP can coherently synthesize conflicting beliefs into a rational consensus. The proposed \texttt{Fed-LLMPrior} algorithm provides a concrete blueprint for implementing these ideas in a federated setting. While challenges in scalability, robustness, and uncertainty quantification remain, this work represents an important step towards a new paradigm of knowledge-driven Bayesian modeling, with the potential to dramatically lower the barrier to entry and expand the reach of Bayesian methods.

\section*{Code Availability}
The code used in this work is available at: \url{https://github.com/YongchaoHuang/llm_prior}

\bibliography{references}
\bibliographystyle{plain}

\newpage
\appendix

\section{Comparison} \label{app:comparison_tables}

\begin{table}[H]
\caption{Comparison of Architectural Blueprints for \texttt{LLMPrior}}
\label{tab:arch_comparison}
\centering
\begin{tabular}{p{2.5cm} p{3.5cm} p{3.5cm} p{3.5cm}}
\toprule
\textbf{Feature} & \textbf{MDN-LLM} & \textbf{NF-LLM} & \textbf{EBM-LLM} \\
\midrule
\textbf{Core Principle} & LLM generates parameters for a Gaussian Mixture Model (GMM). & LLM generates parameters for a sequence of invertible transformations. & LLM generates an unnormalized energy function or symbolic expression. \\
\textbf{Expressiveness} & High. Universal approximator of continuous densities \cite{bishop1994mixture}. & Very High. Can model complex topologies, dependent on flow type and depth \cite{rezende2015variational,papamakarios_normalizing_2021}. & Maximal. Not constrained by a specific parametric form. \\
\textbf{Guaranteed Validity} & Yes, by construction (with softmax on weights and Cholesky on covariances) \cite{gugulothu_mixture_2024}. & Yes, if generated parameters satisfy invertibility constraints \cite{rezende_variational_2015}. & No. Requires an intractable normalization constant ($Z$) to become a valid PDF \cite{goodfellow2016deep}. \\
\textbf{Density Eval.} & Tractable. Involves a simple weighted sum of Gaussian PDFs. & Tractable. Requires inverting the flow and computing Jacobian determinants \cite{berg_sylvester_2019}. & Intractable. Requires computing the partition function $Z$. \\
\textbf{Sampling} & Tractable. Hierarchical sampling: select component, then sample from Gaussian. & Tractable. A single forward pass through the flow from a base sample. & Difficult. Typically requires MCMC-based methods (e.g. Langevin). \\
\textbf{Stability} & Generally stable. Output structure is simple. Potentially mode collapse in GMM. & Can be unstable. Invertibility constraints can be hard to enforce on LLM output. & Depends entirely on the quality of the LLM's regression capabilities \cite{sharlin_context_2025,shojaee_llm-sr_2025}. \\
\textbf{Primary Use Case} & General-purpose, robust prior generation for standard Bayesian workflows. & Priors with highly complex, non-linear dependencies between variables. & Theoretical exploration; likelihood-free inference methods (e.g. ABC). \\
\bottomrule
\end{tabular}
\end{table}

\begin{table}[H]
\caption{Comparison of Prior Aggregation Strategies}
\label{tab:agg_comparison}
\centering
\begin{tabular}{p{3cm} p{3.5cm} p{3.5cm} p{3.5cm}}
\toprule
\textbf{Feature} & \textbf{Parameter Avg. (FedAvg-style)} & \textbf{Centralized LogP} & \textbf{Decentralized LogP (Message Passing)} \\
\midrule
\textbf{Core Principle} & Weighted arithmetic mean of the parameters of prior distributions, e.g. GMM parameters. & Weighted geometric mean of the prior densities, computed on a central server. & Iterative, local geometric mean of prior densities among neighboring agents. \\
\textbf{Theoretical Guarantees} & None. A heuristic that can fail badly with heterogeneous contexts \cite{mcmahan2017communication,ek_federated_2021}. & Externally Bayesian, KL-divergence optimal, preserves log-concavity \cite{carvalho2023bayesian}. & Same as centralized LogP, provided the network reaches a consensus. \\
\textbf{Communication Pattern} & One-shot, clients-to-server. & One-shot, clients-to-server. & Iterative, peer-to-peer among neighbors. \\
\textbf{Robustness (Non-IID)} & Poor. Assumes semantic correspondence of parameters, which is violated by heterogeneity \cite{mcmahan2017communication,ek_federated_2021}. & High. Aggregates at the distribution level, inherently robust to parameterization differences. & High. Local pooling and iterative refinement can lead to robust consensus. \\
\textbf{Comp. Cost (Server)} & Low (simple averaging). & High (requires approximating a product of mixtures). & N/A (no server). \\
\textbf{Comp. Cost (Agent)} & Low (generation only). & Low (generation only). & Medium (generation + iterative updates) \cite{pavlin_multi-agent_2010}. \\
\bottomrule
\end{tabular}
\end{table}

\section{Derivation of the Logarithmic Pool for two Beta Distributions}
\label{app:aggregation_of_two_Betas_derivation}

Here we provide a detailed derivation for the aggregation of two Beta distributions using the Logarithmic Opinion Pool (LogP) with equal weights.

We start with the LogP formula for two agents with equal weights ($w_1 = w_2 = 0.5$):
    \begin{equation*}
        p_{\text{agg}}(\theta) \propto p_1(\theta)^{0.5} \cdot p_2(\theta)^{0.5}
    \end{equation*}

In following derivation, we use the proportional form (kernel) of the Beta PDF \footnote{
The unnormalised density expression defines the shape of the aggregated distribution. It satisfies non-negativity but doesn't integrate to unity. To make it a valid PDF, we must introduce a normalizing constant, $C$, such that:
$p_{\text{agg}}(\theta) = C \cdot p_1(\theta)^{0.5} \cdot p_2(\theta)^{0.5}$
where the constant $C$ is calculated to ensure the total probability is 1:
$C = \frac{1}{\int p_1(\theta)^{0.5} \cdot p_2(\theta)^{0.5} \,d\theta}$.
}. The full PDF is $p(\theta|a,b) = \frac{\Gamma(a+b)}{\Gamma(a)\Gamma(b)} \theta^{a-1}(1-\theta)^{b-1}$. For finding the resulting distribution's family and parameters, we only need the kernel, which captures the dependence on $\theta$:
    \begin{align*}
        p_1(\theta | a_1, b_1) &\propto \theta^{a_1-1}(1-\theta)^{b_1-1} \\
        p_2(\theta | a_2, b_2) &\propto \theta^{a_2-1}(1-\theta)^{b_2-1}
    \end{align*}

Substitute the kernels into the LogP formula:
    \begin{equation*}
        p_{\text{agg}}(\theta) \propto \left( \theta^{a_1-1}(1-\theta)^{b_1-1} \right)^{0.5} \cdot \left( \theta^{a_2-1}(1-\theta)^{b_2-1} \right)^{0.5}
    \end{equation*}

Apply the power rule ($(x^m)^n = x^{mn}$) to each term:
    \begin{equation*}
        p_{\text{agg}}(\theta) \propto \left( \theta^{0.5(a_1-1)}(1-\theta)^{0.5(b_1-1)} \right) \cdot \left( \theta^{0.5(a_2-1)}(1-\theta)^{0.5(b_2-1)} \right)
    \end{equation*}

Combine the terms by adding the exponents ($x^m \cdot x^n = x^{m+n}$):
    \begin{equation*}
        p_{\text{agg}}(\theta) \propto \theta^{0.5(a_1-1) + 0.5(a_2-1)} \cdot (1-\theta)^{0.5(b_1-1) + 0.5(b_2-1)}
    \end{equation*}

Simplify the exponents:
    \begin{align*}
        p_{\text{agg}}(\theta) &\propto \theta^{0.5a_1 - 0.5 + 0.5a_2 - 0.5} \cdot (1-\theta)^{0.5b_1 - 0.5 + 0.5b_2 - 0.5} \\
        p_{\text{agg}}(\theta) &\propto \theta^{0.5(a_1+a_2) - 1} \cdot (1-\theta)^{0.5(b_1+b_2) - 1}
    \end{align*}

Recognize the resulting kernel is the kernel of a new Beta distribution, $\text{Beta}(a_{\text{agg}}, b_{\text{agg}})$, where the parameters are:
    \begin{align*}
        a_{\text{agg}} &= 0.5(a_1+a_2) \\
        b_{\text{agg}} &= 0.5(b_1+b_2)
    \end{align*}

Thus, the formula for the aggregated prior is $p_{\text{agg}}(\theta) \sim \text{Beta}(0.5(a_1+a_2), 0.5(b_1+b_2))$. This confirms that the Logarithmic Pool of Beta distributions is itself a Beta distribution.

\end{document}